\documentclass{article}

\usepackage[preprint]{neurips_2026}

\usepackage[utf8]{inputenc} 
\usepackage[T1]{fontenc}    
\usepackage{hyperref}       
\usepackage{url}            
\usepackage{booktabs}       
\usepackage{amsfonts}       
\usepackage{nicefrac}       
\usepackage{microtype}      
\usepackage{xcolor}         
\usepackage{graphicx}
\usepackage{amsmath}
\usepackage{amssymb}
\usepackage{amsthm}

\usepackage{algorithm}
\usepackage{algorithmic}

\theoremstyle{plain}

\newtheorem{thm}{Theorem}

\theoremstyle{definition}
\newtheorem{defi}{Definition}

\theoremstyle{remark} 

\title{Bayesian Latent Space Models\\ for Graphs Are Misspecified:\\Toward Robust Inference via Generalized Posteriors}

\author{%
  Aldric Labarthe \\
  Centre Borelli, Université Paris-Saclay\\
  Gif-Sur-Yvette, FR 91190 \\
  Department of Computer Science, University of Geneva\\
  Carouge, CH-1227 \\
  \texttt{aldric.labarthe@ens-paris-saclay.fr} \\
}

\begin{document}

\maketitle

\begin{abstract}
Bayesian latent space models offer a principled approach to network representation, but rely on correct specification of both geometry and link function. Real-world networks often violate these assumptions, exhibiting geometric mismatch and structural anomalies that break standard metric properties. We show that such misspecification pushes the data-generating distribution outside the model class, causing Bayesian inference to become overconfident and poorly calibrated. To address this, we propose a generalized posterior framework for random geometric graphs. We introduce Link-Sequential R-SafeBayes, a method that exploits dyadic conditional independence to estimate prequential risk and adaptively tune posterior regularization. Experiments on synthetic and real-world networks demonstrate improved calibration, better link prediction performance, and a reliable criterion for selecting latent geometries across Euclidean, spherical, and hyperbolic spaces.
\end{abstract}

\section{Introduction}
Latent space models have become a cornerstone of network science and graph machine learning, offering to translate graph structures into continuous geometric representations. Since the foundational work of \citet{hoff2002latent}, embedding nodes into a metric space---where connection probabilities decay with pairwise distance---has proven highly effective for modeling homophily, transitivity, and community structure. More recently, the field has recognized the impact of the underlying manifold's geometry, with non-Euclidean spaces, particularly hyperbolic \citep{krioukov2010hyperbolic} and spherical geometries, being leveraged to capture scale-free degree distributions and hierarchical graph topologies. A parallel advancement has been the Bayesian formulation of these models, which provides principled uncertainty quantification for downstream tasks such as link prediction and network reconstruction \citep{newman2018network}.

Despite these successes, standard Bayesian latent space models operate under a perilous assumption: that the true data-generating process perfectly aligns with the chosen parametric model and geometry. In practice, real-world networks are notoriously noisy and inherently violate strict metric axioms. For instance, the presence of dense hubs can exceed the packing capacity of a Euclidean space, while heterophilic connections violate the strict monotonicity of distance-based link functions. In this paper, we argue that Bayesian Random Geometric Graph (RGG) models are almost universally misspecified. 

Under the standard Bayesian paradigm, model misspecification is not benign. When the true data-generating distribution lies outside the pure model class, standard Bayesian updating can lead to overconfidence and inconsistent inference, concentrating the posterior mass on a suboptimal parameter set that yields sub-optimal predictive log-loss \citep{grunwald2017safebayes}. Specifically, we demonstrate theoretically that geometric mismatch forces the Kullback-Leibler minimizer out of the pure model space and strictly into its convex hull. Consequently, any single point estimate or standard posterior sample is compelled to make mispredictions (e.g., assigning zero probability to an observed edge that violates the triangle inequality).

To address misspecification, we propose a robust generalized posterior framework. Because standard marginal likelihoods are unreliable under misspecification, we design \textit{Link-Sequential R-SafeBayes}, an algorithm that sequentially assimilates randomized dyadic blocks to dynamically evaluate prequential risk and select the optimal degree of posterior regularization.

The primary contributions of this work are threefold:
\begin{itemize}
    \item \textbf{Theoretical formalization of geometric misspecification.} We prove that both packing-capacity violations (Theorem~\ref{thm:misspecification_geometry}) and link-function contradictions (Theorem~\ref{thm:misspecification_linkfct}) drive the true network distribution into the convex hull of the model class, formalizing why standard Bayesian RGGs fail.
    \item \textbf{A robust inference framework for networks.} We introduce a generalized $\eta$-posterior for RGGs and develop Link-Sequential R-SafeBayes, adapting the prequential risk framework to conditionally independent dyadic data to safely learn the optimal learning rate.
    \item \textbf{Empirical validation and geometric model selection.} Through extensive experiments on synthetic and real-world networks, we demonstrate that our framework not only guards against overconfidence and improves downstream squared-loss for link prediction, but its prequential log-loss inherently serves as a robust criterion for identifying the most appropriate underlying geometry (Euclidean, Hyperbolic, or Spherical) for a given network.
\end{itemize}

\section{Related Works}

\paragraph{Random graph models and latent geometry.}
Early probabilistic models of networks, such as \cite{erd6s1960evolution} and \cite{gilbert1959random}, assume independence between dyads and ignore higher-order structure. While analytically tractable, these models fail to capture key empirical properties such as clustering and community structure. A major advance came with the introduction of latent space models \citep{hoff2002latent}, where nodes are embedded in a metric space and edge probabilities depend on pairwise distances. In parallel, random geometric graphs (RGG~\cite{penrose_random_2003}) formalized the role of geometry in network formation, providing a natural framework to relate spatial proximity to connectivity. Subsequent work established deep links between geometry and higher-order network properties, including clustering~\citep{krioukov2016clustering} and community emergence~\citep{zuev2015emergence}. Moreover, discrete curvature notions, such as Ricci curvature, have been shown to control the presence and frequency of triangles in graphs~\citep{lin2011ricci, jost2014ollivier}, further reinforcing the interplay between geometry and topology.

\paragraph{Hyperbolic and non-Euclidean embeddings.}
A growing body of work argues that non-Euclidean geometries, in particular hyperbolic spaces, provide a more faithful representation of complex networks. The seminal works of~\citet{krioukov2009curvature, krioukov2010hyperbolic} showed that heterogeneous degree distributions and strong clustering naturally emerge from negative curvature. This perspective has been supported both theoretically and empirically, with further analyses of hyperbolic embeddings~\citep{sala2018representation} and demonstrations that many real-world networks exhibit hyperbolic structure based on Laplacian spectral properties~\citep{smith2019geometry}. Alternative embedding methods have explored metric properties of graphs~\citep{verbeek2014metric}, deterministic constructions in the Poincaré ball~\citep{keller2020hydra, chowdhary2018improved}, and algorithmic benefits such as efficient shortest-path computation. Hyperbolic latent spaces have proven particularly effective for link prediction~\citep{kitsak2020link} and network comparison~\citep{asta2015geom}, while theoretical results confirm that heavy-tailed degree distributions and clustering arise as direct consequences of the hyperbolic metric~\citep{gugelmann2012random}. Beyond Riemannian settings, Lorentzian geometry has also been proposed to embed directed acyclic graphs by leveraging its causal structure~\citep{clough2017embedding}.

\paragraph{Bayesian latent space models and uncertainty.}
The probabilistic formulation of latent space models was pioneered by~\citet{hoff2002latent}, who introduced a Bayesian framework with a logistic link function in Euclidean space, later extended for scalability by~\citet{raftery2012fast}. However, most embedding methods—Bayesian or otherwise—return point estimates and largely ignore uncertainty, despite its critical impact on downstream tasks such as link prediction. Early Bayesian approaches to random geometric graphs and network reconstruction~\citep{newman2018network} emphasized the presence of measurement errors in network data, an idea further developed to account for heterogeneous noise~\citep{young2020bayesian, peixoto2018reconstructing} and community detection~\citep{vanderpas2018bayescommdetect}. 

A key challenge in Bayesian latent space models is non-identifiability: latent coordinates are only defined up to isometries (e.g., Procrustes transformations). This issue has been addressed through anchoring strategies~\citep{papamichalis2021latent} or by explicitly modeling symmetries and multimodal posteriors, particularly in hyperbolic settings~\citep{lizotte2025symmetry}. More broadly, Bayesian methods have also been used to infer geometric properties such as manifold curvature and dimension, even outside strict latent position models~\citep{lubold2023identifying}.

\paragraph{Limitations and gap.}
Despite these advances and the growing interest in Bayesian latent space models, existing approaches largely assume that the chosen geometric space and link function are correctly specified. In practice, however, real-world networks are generated by mechanisms that are only imperfectly captured by any single latent space model. While uncertainty in parameters has been studied, uncertainty induced by model misspecification—particularly in the joint choice of geometry and link function—remains largely overlooked. This gap motivates our work, which explicitly addresses misspecification in random geometric graph models and proposes a Bayesian framework robust to such deviations.

\section{Methods}
\begin{figure}[!ht]
\centering
\includegraphics[width=\linewidth]{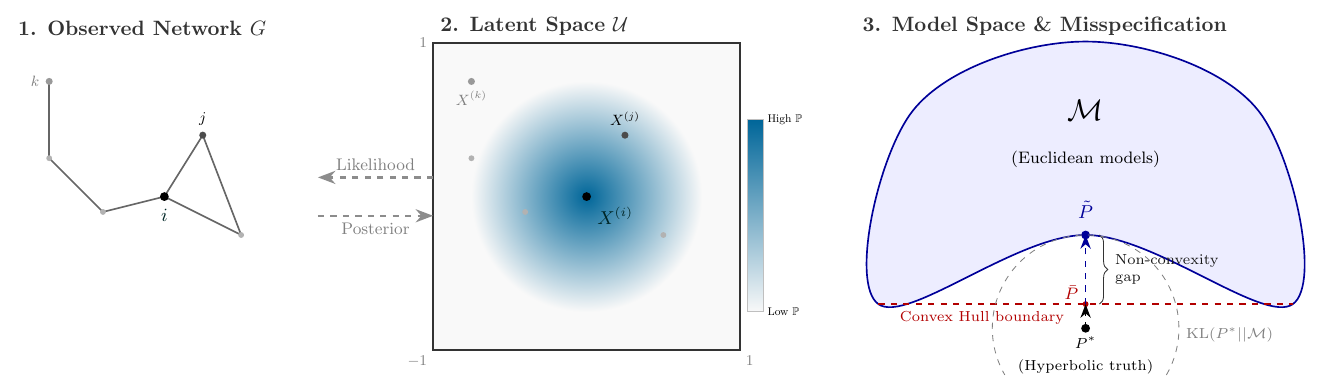}
\caption{\textbf{Conceptual framework of geometric misspecification in Bayesian RGGs.} (\textbf{1} and \textbf{2}) An observed network $G$ and its Latent positions $X$ in space $\mathcal{U}$ inferred via the likelihood $p(G|X)$ and the posterior $p(X|G)$. (\textbf{3}) Theoretical illustration of model misspecification: the model set $\mathcal{M}$ (e.g., Euclidean RGGs) is inherently non-convex. When the ground truth $P^*$ (e.g., hyperbolic) lies outside $\mathcal{M}$, the distribution $\tilde{P}$ that minimizes KL-divergence resides strictly in $\operatorname{conv}(\mathcal{M})$ rather than in $\mathcal{M}$ (Theorems \ref{thm:misspecification_geometry} and \ref{thm:misspecification_linkfct}). }
\label{fig:misspecification_framework}
\end{figure}
\subsection{Misspecification of existing models}

\begin{defi}[The General RGG Model $\mathcal{M}$]
Let $G = (V, E)$ be an undirected, unweighted graph where $V = \{1, \dots, N\}$. The observable data space is $\mathcal{Y} = \{0, 1\}^{\binom{N}{2}}$, representing the adjacency matrix $A$.
We define a parametric model $\mathcal{M} = \{ P_{\theta} : \theta \in \Theta \}$ governed by:
\begin{enumerate}
    \item \textbf{A Metric Space:} $(\mathcal{Z}, d_{\mathcal{Z}})$, where latent coordinates are $Z = \{z_1, \dots, z_N\} \in \mathcal{Z}^N$.
    \item \textbf{A Link Function:} $f: \mathbb{R}_{\geq 0} \to [0, 1]$ is a strictly monotonically decreasing function.
    \item \textbf{The Likelihood:} Assuming conditional independence of dyads given $Z$:
    $$
    P_{\theta}(A) = \prod_{i<j} f(d_{\mathcal{Z}}(z_i, z_j))^{a_{ij}} (1 - f(d_{\mathcal{Z}}(z_i, z_j)))^{1 - a_{ij}}
    $$
\end{enumerate}
where the parameter vector is $\theta = (Z, f)$.
\end{defi}

Let $P^*$ be the true data-generating distribution such that $P^* \notin \mathcal{M}$. Let $\text{conv}(\mathcal{M})$ denote the convex hull of $\mathcal{M}$, consisting of all finite mixture distributions $P_{\text{mix}} = \sum_{k=1}^K w_k P_{\theta_k}$ where $\sum w_k = 1$. 

In this paper, we claim that this class of models is misspecified in all real-world settings, and that this misspecification can affect the quality of the learned representation. Misspecification comes from two sources: the geometry of the space and the link function. 

The model is misspecified in the \citet{grunwald2017safebayes} sense (leading to standard Bayesian inconsistency) if the distribution $\tilde{P}$ that minimizes the Kullback-Leibler divergence from $P^*$ lies strictly outside the pure model set:
$$
\tilde{P} = \underset{P \in \text{conv}(\mathcal{M})}{\operatorname{argmin}} D_{KL}(P^* \parallel P) \implies \tilde{P} \notin \mathcal{M}
$$

\begin{thm}\label{thm:misspecification_geometry}
If the true distribution $P^*$ generates a geometry whose volume growth exceeds the packing capacity of the model metric space $(\mathcal{Z}, d_{\mathcal{Z}})$, the KL-minimizer resides strictly in the convex hull, $\tilde{P} \notin \mathcal{M}$.
\end{thm}

\begin{thm}\label{thm:misspecification_linkfct}
If the true distribution $P^*$ generates edges that violate the strict monotonicity of the model's link function $f$, the KL-minimizer resides strictly in the convex hull, $\tilde{P} \notin \mathcal{M}$.
\end{thm}

Both theorems are formally proven in \textbf{Appendix~\ref{app:proofs}}. To build intuition for how algebraic misspecification drives the true data-generating distribution into the convex hull, consider a scenario where the ground truth $P^*$ generates edges via a soft-threshold (sigmoid) link function, but the model class $\mathcal{M}$ is rigidly restricted to hard-threshold (step function) distance cutoffs. A single pure state in $\mathcal{M}$ must commit to a strictly discrete radius $R$, predicting connections with absolute certainty inside $R$ and complete impossibility outside $R$. Because $P^*$ will inevitably generate topological noise---edges where $d > R$ and non-edges where $d \le R$---any pure state suffers catastrophic, unbounded log-loss penalties. Conversely, a distribution in the convex hull $\text{conv}(\mathcal{M})$ can easily construct a mixture over varying radii $R_k$ with corresponding weights $w_k$. This ensemble algebraically behaves as a descending staircase, forming a discretized but highly accurate approximation of the true continuous sigmoid curve. Because this mixture completely avoids infinite penalties and gracefully traces the true geometric decay profile, it achieves a strictly lower Kullback-Leibler divergence to $P^*$ than any single step function. Consequently, the KL-minimizer $\tilde{P}$ resides strictly outside the pure model set ($\tilde{P} \notin \mathcal{M}$).

Given that the KL-minimizer $\tilde{P}$ resides outside the pure model set $\mathcal{M}$, standard Bayesian inference is guaranteed to fail. As sample size grows, the standard posterior will concentrate on a single convex combination of pure states rather than on the best in-model approximation of $P^*$.

\subsection{Safe Bayesian Estimation Under Misspecification}
To circumvent this issue, we design an analogue of the \textit{SafeBayes} algorithm \citep{grunwald2017safebayes} on graphs. To do so, we first restrict ourselves to a specific RGG model (while keeping in mind that our arguments are valid no matter this specific choice), then we describe our SafeBayes variant, and finally we conceive a test bench to empirically verify the reality of misspecification.
 
\subsubsection{Our Latent Space Model}
We build upon the latent space model formulation of \citet{hoff2002latent}. We model dyads $A_{ij}$ as conditionally independent Bernoulli variables given latent positions $z_i, z_j \in \mathcal{Z}$. We define a soft-threshold Random Geometric Graph (RGG) connection probability:
\begin{equation}
    P(A_{ij} = 1 \mid z_i, z_j, r, \alpha) = \sigma\big(\alpha (r - d_{\mathcal{Z}}(z_i, z_j))\big)
\end{equation}
where $d_{\mathcal{Z}}(\cdot, \cdot)$ is the geodesic distance, $r \sim \mathcal{HN}(r_{\text{scale}})$ is a global threshold radius, $\alpha$ scales the strictness of the threshold, and $\sigma(\cdot)$ is the sigmoid function. 

To prevent the posterior from concentrating on the wrong solution, we use a generalized $\eta$-posterior:
\begin{equation}
    \pi_{\eta}(Z, r \mid A) \propto \pi(Z) \pi(r) \prod_{i < j} P(A_{ij} \mid z_i, z_j, r)^{\eta}
\end{equation}
where $\eta \in (0, 1]$ is the learning rate. By setting $\eta < 1$, we flatten the posterior, effectively simulating a mixture over pure states in $\text{conv}(\mathcal{M})$ and guarding against unbounded log-loss penalties.

\subsubsection{Link-Sequential R-SafeBayes}

Determining the optimal degree of posterior flattening $\eta$ cannot be done via standard marginal likelihoods, as those are already biased by the model's misspecification. Instead, we adapt the R-SafeBayes algorithm to dynamically select $\eta$ by sequentially evaluating prequential risk \citep{dawid1984prequential}. 

Safe Bayes relies on the independence of observations. While edges in a network are unconditionally dependent, our latent space model guarantees that dyads are \textit{conditionally independent} given $Z$. To leverage this without violating graph structure or leaking data, we partition the dyads into randomized, disjoint blocks $B_1, \dots, B_K$. 

By sequentially predicting connections in unseen blocks, we measure exactly how severely geometric and link-function misspecifications are degrading out-of-sample predictions. For a given $\eta$, we first fit the $\eta$-posterior on the currently observed blocks, then we compute the posterior predictive distribution for the target block $B_k$ and finally we evaluate cumulative predictive log-loss across blocks, capturing the penalties induced by mismatches.

\begin{algorithm}[H]
\caption{Link-Sequential R-SafeBayes for Misspecified Network Models}
\label{alg:rsafebayes}
\begin{algorithmic}[1]
\REQUIRE Adjacency matrix $A \in \{0,1\}^{N \times N}$, Candidate geometries $\mathcal{G}$, Candidate rates $\mathcal{H}$, Initial fraction $\rho$, Blocks $K$
\STATE Extract upper-triangular dyads $\mathcal{D} = \{ (i,j) \mid 1 \leq i < j \leq N \}$
\STATE Randomly shuffle and partition $\mathcal{D}$ into $B_0$ and test blocks $B_1, \dots, B_K$
\FOR{\textbf{each} geometry $\mathcal{Z} \in \mathcal{G}$}
    \FOR{\textbf{each} learning rate $\eta \in \mathcal{H}$}
        \STATE Initialize cumulative risk $\mathcal{R}_{\mathcal{M}}(\eta) \leftarrow 0$, $D_{\text{train}} \leftarrow B_0$
        \FOR{$k = 1$ \textbf{to} $K$}
            \STATE $S \leftarrow \text{Sample}(\pi_{\eta}(Z, r \mid D_{\text{train}} \text{ on } \mathcal{Z}))$ \textit{// NUTS sampling}
            \FOR{\textbf{each} $(i,j) \in B_k$}
                \STATE $\hat{p}_{ij} \leftarrow \frac{1}{|S|} \sum_{(Z, r) \in S} \sigma\big(\alpha (r - d_{\mathcal{Z}}(z_i, z_j))\big)$
            \ENDFOR
            \STATE $L_k \leftarrow - \frac{1}{|B_k|} \sum_{(i,j) \in B_k} \left[ A_{ij} \log \hat{p}_{ij} + (1 - A_{ij}) \log (1 - \hat{p}_{ij}) \right]$
            \STATE $\mathcal{R}_{\mathcal{M}}(\eta) \leftarrow \mathcal{R}_{\mathcal{M}}(\eta) + L_k$
            \STATE $D_{\text{train}} \leftarrow D_{\text{train}} \cup B_k$ \textit{// Prequential update}
        \ENDFOR
    \ENDFOR
    \STATE $\eta^*_{\mathcal{M}} \leftarrow \arg\min_{\eta \in \mathcal{H}} \mathcal{R}_{\mathcal{M}}(\eta)$
\ENDFOR
\RETURN Optimal learning rates $\eta^*_{\mathcal{M}}$ and predictive risks for all $\mathcal{Z} \in \mathcal{G}$
\end{algorithmic}
\end{algorithm}

The cumulative log-loss acts as the prequential posterior risk within the SafeBayes framework, sequentially quantifying the predictive penalty of the model as it assimilates new network blocks. The prequential posterior risk measures the overall model quality, directly reflecting how effectively the learned representation captures the true underlying graph structure and generalizes to unseen data.

While the prequential log-loss drives the learning rate selection, we simultaneously track the cumulative squared-loss (Brier score) across the testing blocks. For practical graph representation, this squared-loss is the primary metric we care about, as it directly dictates the model's reliability for link prediction and downstream predictive tasks. Crucially, as highlighted in the Safe Bayes literature, performing standard Bayesian inference ($\eta = 1$) under severe misspecification forces the model into overconfident estimations that translates to an artificially inflated squared-loss.

\subsubsection{Experimental Settings}

Our experimental pipeline has three objectives. First, it aims at validating the misspecification on empirical datasets by observing the need for $\eta < 1$. Second, it will estimate the square loss gains induced by the $\eta$-regularization for downstream tasks. Third, it will leverage the model-selection criterion of the log-loss and the $\eta$ learning rate to assess whether Safe Bayes is able, as a byproduct of its regularization, to determine the best geometric hypothesis (hyperbolic, spherical, or Euclidean) for the given empirical network.

\paragraph{Candidate geometries.} 
We evaluate the framework across three 2D Riemannian manifolds: (i) \textbf{Euclidean} ($\mathbb{E}^2$), which is highly susceptible to volume-packing misspecification in hierarchical graphs; (ii) \textbf{Hyperbolic} ($\mathbb{H}^2$), capturing exponential volume expansion, implemented via the Lorentz model in $\mathbb{R}^3$ with shielded Taylor expansions to prevent HMC singularities; and (iii) \textbf{Spherical} ($\mathbb{S}^2$), which enforces strict global distance bounds. 

\paragraph{Datasets and evaluation.} 
We benchmark on synthetic and empirical networks. The synthetic topologies isolate specific geometric biases: a 5-community SBM (matching $\mathbb{E}^2$), a Poincaré network ($\mathbb{H}^2$), a spherical network ($\mathbb{S}^2$), and a Core-Periphery model designed to violate metric axioms. All synthetic geometries are generated with a hard-threshold RGG model to incur link function mismatch. We additionally evaluate six empirical networks from KONECT \citep{konect} (Karate, Les Misérables, Dolphins, Moreno train, Iceland, Windsurfers) containing inherent structural noise uncapturable by pure-state models. To trace the posterior's descent into the convex hull, R-SafeBayes evaluates a discrete grid of learning rates $\eta \in \{0.1, 0.2, \dots, 1.0\}$ across all dataset-geometry pairs.

\paragraph{Sampling and Alignment}
Inference is conducted via the No-U-Turn Sampler (NUTS) using JAX and NumPyro. Because RGG likelihoods are fully invariant under isometries, the un-tempered posterior is notoriously multimodal. Rather than artificially breaking these symmetries (which creates false constraints), we allow NUTS to explore freely. The resulting predictive distribution $\hat{p}_{ij}$ remains strictly invariant to translation and rotation. For downstream visualization, we utilize post-hoc Orthogonal Procrustes alignment. To maintain numerical stability while sampling the tempered posterior, we apply mathematically equivalent log-sigmoid reformulations ($\log(1 - \sigma(x)) = \log(\sigma(-x))$) and parallel \texttt{pmap} dispatch, ensuring robust convergence across all misspecified regimes.

\section{Results}
\subsection{The Pervasiveness of Geometric Misspecification}
As theorized in our framework, the assumption that the true data-generating distribution resides within the pure model class ($P^* \in \mathcal{M}$) is systematically violated across empirical network topologies. Table~\ref{tab:results_summary} and Figure~\ref{fig:losses} reveal a striking consistency: across all real-world networks the optimal learning rate determined by Link-Sequential R-SafeBayes is less than one ($\eta^* < 1.0$). Even the synthetic datasets feature this correction due to link function mismatch. This confirms our theoretical assertion (Theorems~\ref{thm:misspecification_geometry} and~\ref{thm:misspecification_linkfct}): fundamental topological noise, such as packing capacity violations, local heterophily or misspecified link function, invariably drives the information-theoretic projection of $P^*$ into the convex hull of the model class. Consequently, standard Bayesian inference ($\eta=1.0$) is intrinsically misspecified for these models and fails to achieve predictive consistency, necessitating posterior flattening to simulate a convex mixture of pure states.

\begin{figure}[!h]
\centering
\includegraphics[width=\linewidth]{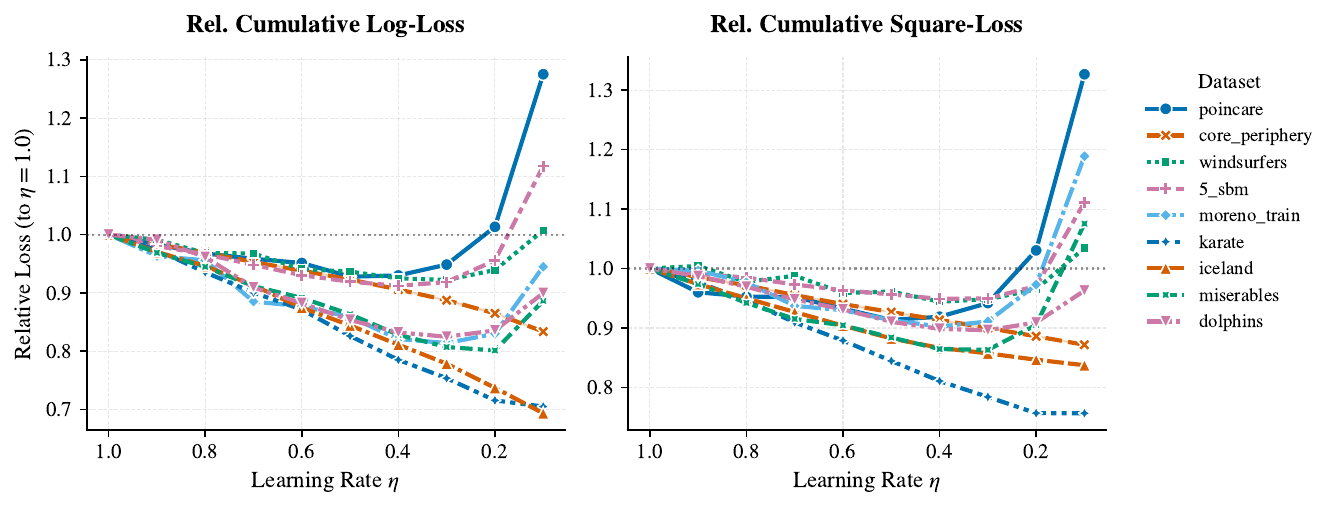}
\caption{\textbf{Predictive Performance vs. Likelihood Weighting $\eta$.} Relative cumulative log-loss (left) and square-loss (right) across networks, normalized to the standard Bayesian baseline ($\eta=1.0$). }
\label{fig:losses}
\end{figure}

To demonstrate that the optimal learning rate $\eta^*$ reflects the intrinsic geometric misspecification of the network rather than an artifact of the prequential partitioning scheme, we conduct a sensitivity analysis on the Link-Sequential R-SafeBayes hyperparameters. Figure \ref{fig:hyperparameters} illustrates the prequential log-loss trajectories on the Dolphins network (Euclidean model) across independent ablations of the evaluation block count ($K$), the initial training fraction ($p_{train}$), and randomized data orderings. While varying $K$ naturally scales the absolute cumulative risk—as log-loss penalties are integrated over different sequential horizons—the global minimum remains strictly anchored at $\eta^* = 0.3$. Similarly, adjusting the initial training fraction shifts the baseline loss but preserves the identical learning rate minimum. Finally, the stability across randomized shufflings (right panel) demonstrates that the prequential risk surface is invariant to dyadic ordering, with the mean and confidence interval ($\pm 1$ standard deviation) converging tightly at the same optimal $\eta$. This invariant behavior confirms that the necessity of posterior flattening ($\eta < 1$) is a robust, structural imperative driven by the geometric mismatch formalized in Theorems 1 and 2, validating the diagnostic reliability of our framework regardless of the specific sequential evaluation schedule.

\begin{figure}[!h]
\centering
\includegraphics[width=\linewidth]{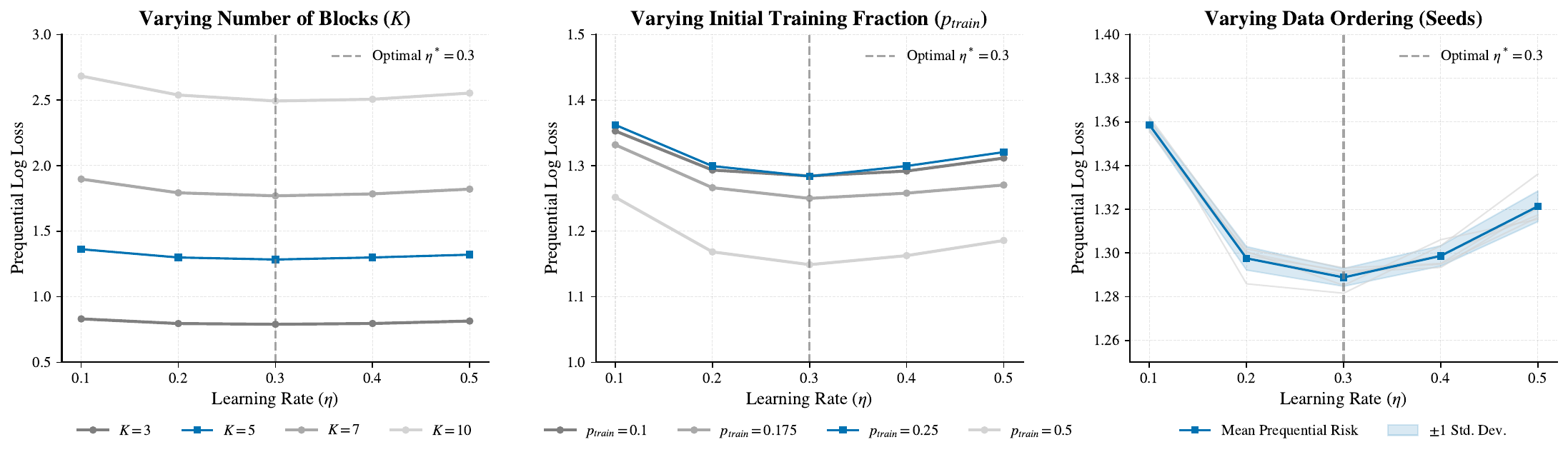}
\caption{\textbf{Robustness of Link-Sequential R-SafeBayes to prequential hyperparameters and data ordering (Dolphins network).} \textit{(Left)} Varying the number of evaluation blocks $K$. \textit{(Middle)} Adjusting the initial training fraction $p_{train}$. \textit{(Right)} Individual empirical trajectories across six randomized dyadic shufflings (grey) and their mean (blue) with $\pm 1$ standard deviation.}
\label{fig:hyperparameters}
\end{figure}

\subsection{Mitigating Overconfidence and Improving Predictive Risk}
The dynamic regularization imposed by R-SafeBayes directly translates into substantial downstream performance gains, evaluated via both predictive log-loss and cumulative squared-loss (Brier score). As illustrated in Figure~\ref{fig:overfitting}, standard Bayesian Random Geometric Graphs (RGGs) reliably fall into an overfitting trap, heavily optimizing for in-sample dyadic reconstruction at the direct expense of out-of-sample predictive validity. By sequentially assimilating data blocks to tune $\eta$, SafeBayes successfully escapes this degenerate concentration, identifying a global minimum that yields, in the Euclidean case, an average $15.6\%$ reduction in log-loss and an $11.6\%$ reduction in squared-loss.

\begin{figure}[!ht]
\centering
\includegraphics[width=\linewidth]{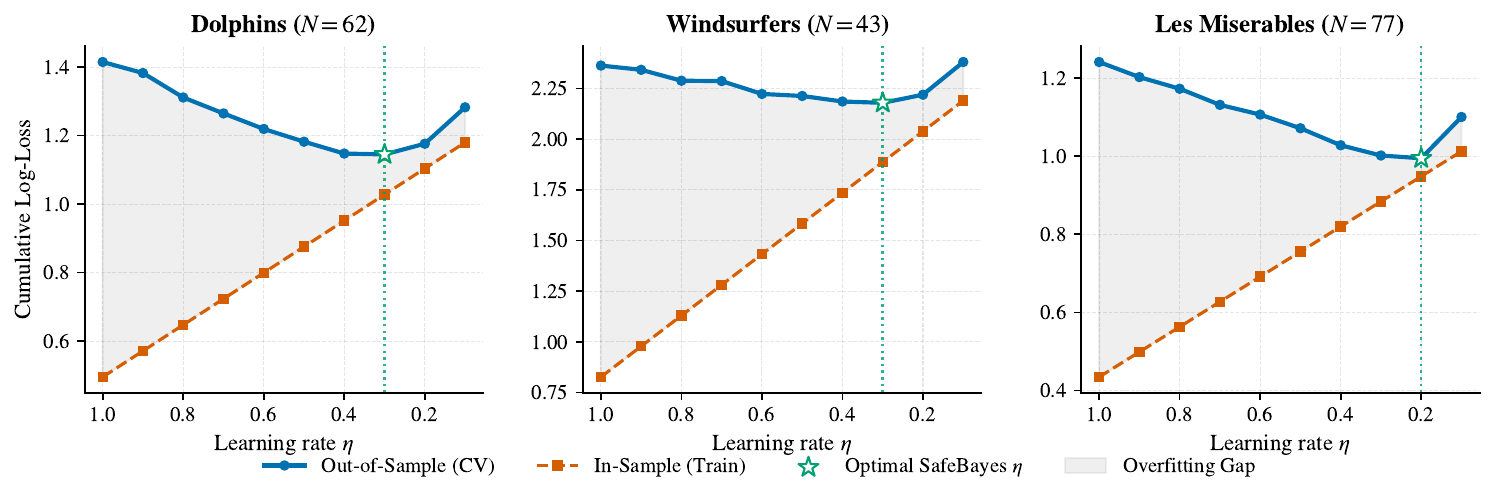}
\caption{\textbf{SafeBayes Overfitting correction} evaluated in the Euclidean case. The shaded area represents the overfitting gap where Standard Bayes ($\eta=1.0$) minimizes in-sample error at the cost of sub-optimal out-of-sample prediction loss. SafeBayes ($\eta < 1.0$) identifies the global minimum, achieving an average $15.6\%$ reduction in log. loss and $11.6\%$ in sq. loss.}
\label{fig:overfitting}
\end{figure}

Crucially, these performance gains do not emerge uniformly across all predictions; rather, they stem from the mitigation of catastrophic predictive failures. Figure~\ref{fig:improvements_by_values} dissects the distribution of these predictive improvements, revealing that standard Bayes forces the model into overconfident, erroneous density estimations on topologically noisy dyads. Under misspecification, a pure metric state is forced to make "sure-but-wrong" predictions (e.g., assigning negligible probability to an observed edge that violates the triangle inequality). By flattening the posterior, SafeBayes adopts a defensive posture. While it incurs a minor regularization tax on easily predictable, low-error dyads, this cost is vastly eclipsed by the massive reduction in predictive penalties in the high-error regime ($\text{Error} > 0.8$). By effectively recognizing its own epistemic uncertainty under structural mismatch, the model avoids unbounded loss penalties.

\begin{figure}[!ht]
\centering
\includegraphics[width=\linewidth]{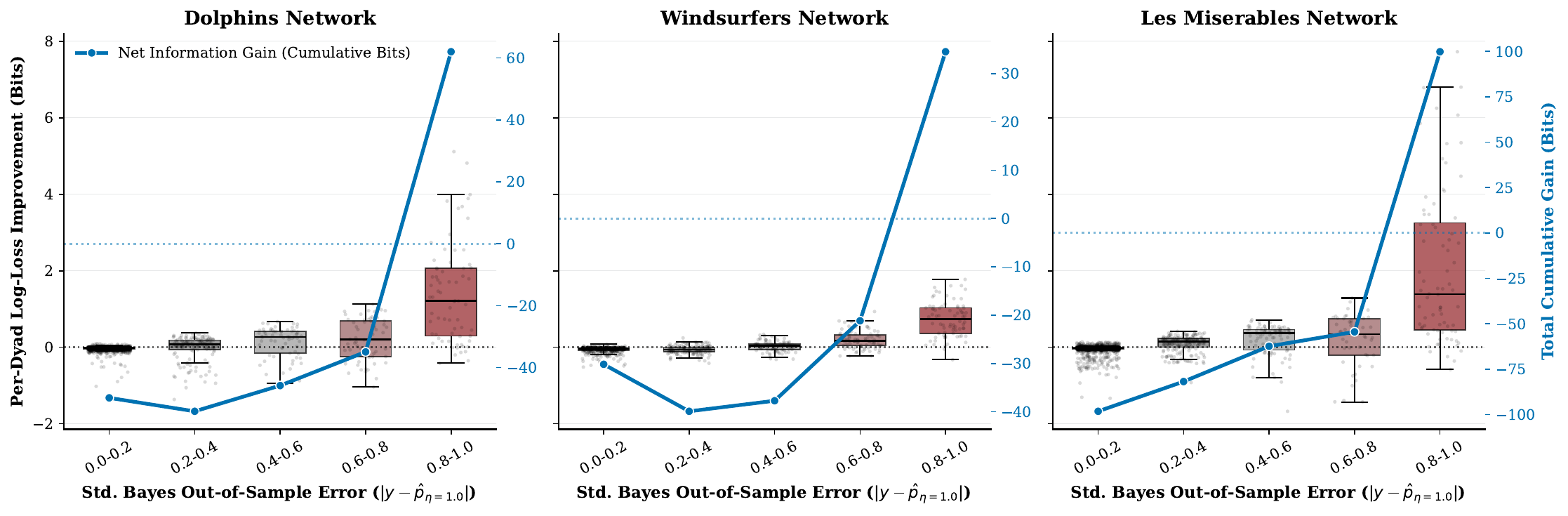}
\caption{\textbf{Distribution of Predictive Improvement.} SafeBayes mitigates catastrophic ``Sure-but-Wrong'' predictions inherent in Euclidean latent space models. While a minor regularization tax is observed in low-error dyads, the cumulative information gain (blue line) is driven by massive penalty reductions in the high-error regime ($\text{Error} > 0.8$).}
\label{fig:improvements_by_values}
\end{figure}

\subsection{Robust Geometric Selection via Prequential Risk}
Beyond guarding against overconfidence, the prequential log-loss evaluated by R-SafeBayes serves as a highly reliable, unsupervised model-selection criterion for uncovering a network's underlying latent manifold. By sequentially measuring the out-of-sample predictive penalties incurred by each geometry, the framework distinguishes between Euclidean, Spherical, and Hyperbolic spaces without relying on heavily biased marginal likelihoods. 

Table~\ref{tab:results_summary} demonstrates the efficacy of this criterion. The framework correctly identifies the true underlying geometries of our synthetic benchmarks, matching the Spherical dataset to $\mathbb{S}^2$, the Poincaré model to $\mathbb{H}^2$ and the $5$-SBM to $\mathbb{E}^2$, even if the link function was misspecified. Most notably, this selection criterion exhibits profound robustness to isometric variations and specific coordinate representations. Evaluated on the synthetic \textit{poincare} dataset—which generates connections according to the Poincaré disk metric—Link-Sequential R-SafeBayes conclusively selects the Lorentz Model (Hyperboloid) as the optimal geometry over both Spherical and Euclidean alternatives. Even though the Hyperboloid and Poincaré disk are distinct isometries (leading to disparate Hamiltonian dynamics and inference pathways), the prequential risk correctly recognizes the fundamental signature of negative curvature and exponential volume expansion. This proves that the SafeBayes log-loss intrinsically recovers the optimal global geometry based on volumetric scaling properties, rendering it a uniquely robust tool for manifold selection in network science.

\section{Limitations}
\paragraph{Computational complexity.}
The most significant limitation of our approach is its computational cost. Standard Bayesian inference for latent space models already scales at $\mathcal{O}(N^2)$ with respect to the number of nodes due to the dense dyadic likelihood evaluation. The prequential nature of our R-SafeBayes algorithm strictly exacerbates this bottleneck. To determine the optimal learning rate $\eta$, the sampler (NUTS) must be executed repeatedly across a discrete grid of candidate rates $\mathcal{H}$, multiple geometric candidate spaces $\mathcal{G}$, and sequentially across $K$ data blocks. While parallel hardware dispatch accelerates this process, the overall $\mathcal{O}(K \cdot |\mathcal{H}| \cdot |\mathcal{G}| \cdot N^2)$ complexity currently restricts the framework's direct application to massive networks. Scaling this robust inference framework to networks with hundreds of thousands of nodes will likely require adapting Generalized Variational Inference (GVI) to the prequential SafeBayes objective, a task proven to be difficult as GVI has been reported inaccurate in non-Euclidean Bayesian inference \citep{lizotte2025symmetry}.

\paragraph{Sampling dynamics on non-Euclidean manifolds.}
A second critical limitation stems from the inherent difficulty of MCMC sampling on complex geometries, particularly when the chosen manifold is heavily misspecified for the observed network. While tempering the posterior via the learning rate $\eta$ actively aids sampling dynamics by flattening the target distribution and easing traversal between modes, Markov chains can still exhibit pathological behaviors, during inference on difficult or heavily mismatched geometries. Some of our runs yielded non-conclusive convergence diagnostics, including inflated $\hat{R}$ statistics and low bulk effective sample sizes (ESS Bulk). Despite our use of the Lorentz model and shielded Taylor expansions, the broader Bayesian literature still lacks robust, universally stable sampling techniques tailored to unconstrained exploration on non-Euclidean manifolds.

\section{Conclusion}

In this paper, we formalize the inherent structural misspecification of Bayesian latent space models applied to empirical graphs. We demonstrate that fundamental topological properties of real-world networks---specifically, volumetric scaling regimes that exceed the packing capacity of the chosen Riemannian manifold and local heterophily that violate the strict monotonicity of the link function---preclude the existence of a well-specified pure state. Under these conditions, we prove that the information-theoretic projection of the true data-generating measure resides strictly within the convex hull of the parametric family ($\tilde{P} \notin \mathcal{M}$). Consequently, standard Bayesian updating fails to achieve predictive consistency; rather than recovering a robust geometric representation, the posterior asymptotically concentrates on degenerate pure states that incur unbounded Kullback-Leibler divergence penalties, leading to severe miscalibration and overconfident density estimation on out-of-sample dyads.

To mitigate these pathological inference dynamics, we deploy a generalized $\eta$-posterior, demonstrating that likelihood tempering operates as an information-theoretic necessity rather than a mere heuristic. By flattening the target measure, the $\eta$-posterior effectively simulates a convex mixture that absorbs structural contradictions without incurring infinite loss. Furthermore, because standard marginal likelihoods are heavily biased under $\tilde{P} \notin \mathcal{M}$, we introduced Link-Sequential R-SafeBayes. This algorithm leverages the conditional independence of dyads to evaluate prequential risk across randomized blocks, offering a statistically sound mechanism to quantify geometric mismatch and dynamically regularize the posterior without data leakage.

Our empirical evaluations underscore the severity of this issue: across multiple real-world topologies, standard Bayesian inference routinely collapses into overconfidence, whereas our prequential framework recovers predictive calibration and significantly improves generalization for downstream link prediction. Crucially, the necessity of such aggressive posterior regularization highlights a broader vulnerability within the graph representation learning literature. The routine deployment of uncalibrated, pure-state geometric models risks encoding fundamental artifacts into latent embeddings. While our adaptation of SafeBayes provides a rigorous diagnostic and mitigation strategy, resolving the underlying computational bottlenecks and pathological sampling dynamics on non-Euclidean manifolds remains an open challenge. Ultimately, this work cautions against the uncritical assumption of geometric specifiedness, urging the community toward inherently robust, uncertainty-aware network inference paradigms.

\bibliographystyle{abbrvnat}
\bibliography{bibliography}

\clearpage
\appendix

\section{Proofs}
\label{app:proofs}

\begin{proof}[Proof of Theorem \ref{thm:misspecification_geometry}]
Assume, for the sake of contradiction, that $\tilde{P} \in \mathcal{M}$. Therefore, there exists a single pure state $\theta^* = (Z^*, f^*)$ that strictly minimizes the expected negative log-likelihood (log-loss):
$$
\mathbb{E}_{A \sim P^*}[-\log P_{\theta^*}(A)]
$$

Let $P^*$ assign probability $1$ to a Star Graph $K_{1, M}$ where node $0$ is the hub and nodes $\{1, \dots, M\}$ are independent leaves. Thus $a_{0i} = 1$ and $a_{ij} = 0$ for all $i, j \in \{1, \dots, M\}$.

To minimize the log-loss of the true edges, the pure state must place all leaves within some small radius $R$ of the hub to maximize $f^*(d_{\mathcal{Z}}(z_0, z_i))$.

Assume the chosen metric space $(\mathcal{Z}, d_{\mathcal{Z}})$ is Euclidean $\mathbb{R}^D$. By the geometric definition of packing density (the Kissing Number bound), there exists a finite maximum number of points $C(D)$ that can be placed around a center while maintaining a mutually bounded separation. 

For any $M \gg C(D)$, by the Pigeonhole Principle and the Triangle Inequality, there must exist at least two leaves $j$ and $k$ that are forced arbitrarily close to each other:
$$
d_{\mathcal{Z}}(z_j, z_k) \le \epsilon
$$

Because $f^*$ is strictly monotonically decreasing, $f^*(\epsilon)$ must be bounded away from $0$ (e.g., $f^*(\epsilon) \ge 1 - \delta$). However, under $P^*$, $a_{jk} = 0$ almost surely. The log-loss penalty for this non-edge is $-\log(1 - f^*(\epsilon))$, which explodes to infinity as $\epsilon \to 0$.

Now, construct a mixture distribution $P_{\text{mix}} \in \text{conv}(\mathcal{M})$:
$$
P_{\text{mix}} = \frac{1}{M} \sum_{k=1}^M P_{\theta_k}
$$
where each pure component $P_{\theta_k}$ maps leaf $k$ perfectly to distance $0$ from the hub, and pushes all other $M-1$ leaves to infinity. 

Under $P_{\text{mix}}$, two leaves $j$ and $k$ are \textit{never} close to the hub simultaneously. Therefore, the distance $d_{\mathcal{Z}}(z_j, z_k) = \infty$ in every mixture component, yielding $P_{\text{mix}}(a_{jk}=1) = 0$. The infinite false-positive penalty is completely averted, while the true edges retain a likelihood of $\frac{1}{M}$.

Because the log-loss of $P_{\text{mix}}$ is strictly bounded and the log-loss of $P_{\theta^*}$ is driven to infinity by the geometric packing contradiction, we have:
$$
D_{KL}(P^* \parallel P_{\text{mix}}) < D_{KL}(P^* \parallel P_{\theta^*})
$$
This contradicts the assumption that $\tilde{P} \in \mathcal{M}$.
\end{proof}

\begin{proof}[Proof of Theorem \ref{thm:misspecification_linkfct}]
Assume, for the sake of contradiction, that $\tilde{P} \in \mathcal{M}$. Therefore, there exists a pure state $\theta^* = (Z^*, f^*)$ that minimizes the expected log-loss.

Let the geometry $(\mathcal{Z}, d_{\mathcal{Z}})$ have infinite packing capacity, isolating the error entirely to the link function. Let $P^*$ generate a strict Bipartite Graph (heterophily). Nodes are partitioned into sets $U$ and $V$. Edges strictly exist between $U$ and $V$, and $a_{u_1 u_2} = 0$ for all $u_1, u_2 \in U$.

Select two nodes $u_1, u_2 \in U$ that both connect to a single node $v \in V$ with probability $1$. To minimize the log-loss for the true edges $(u_1, v)$ and $(u_2, v)$, the pure state $Z^*$ must map them closely:
$$
d_{\mathcal{Z}}(z_{u_1}, z_v) \le \epsilon \quad \text{and} \quad d_{\mathcal{Z}}(z_{u_2}, z_v) \le \epsilon
$$

By the metric axioms of $(\mathcal{Z}, d_{\mathcal{Z}})$, the Triangle Inequality dictates:
$$
d_{\mathcal{Z}}(z_{u_1}, z_{u_2}) \le d_{\mathcal{Z}}(z_{u_1}, z_v) + d_{\mathcal{Z}}(z_v, z_{u_2}) \le 2\epsilon
$$

Because $f^* \in \mathcal{F}$ is strictly monotonically decreasing, if it correctly assigns high probability to distance $\epsilon$, it is mathematically compelled to assign high probability to distance $2\epsilon$. Thus, $P_{\theta^*}(a_{u_1 u_2} = 1)$ is large. But under the true bipartite distribution $P^*$, $a_{u_1 u_2} = 0$. The pure state NLL suffers massive catastrophic penalties for false-positive intra-group predictions.

Now, construct a mixture $P_{\text{mix}} = \frac{1}{2}P_{\theta_1} + \frac{1}{2}P_{\theta_2} \in \text{conv}(\mathcal{M})$:
\begin{itemize}
    \item In $P_{\theta_1}$, place $u_1$ at distance $\epsilon$ from $v$, but push $u_2$ to infinity.
    \item In $P_{\theta_2}$, place $u_2$ at distance $\epsilon$ from $v$, but push $u_1$ to infinity.
\end{itemize}

In this mixture, $u_1$ and $u_2$ are never in the same spatial vicinity in any single model component. The mixture probability $P_{\text{mix}}(a_{u_1 u_2}=1) \to 0$, averting the structural false-positive penalty, while preserving a $\frac{1}{2}$ baseline probability for the true bipartite connections. 

Therefore, $D_{KL}(P^* \parallel P_{\text{mix}}) < D_{KL}(P^* \parallel P_{\theta^*})$, contradicting that $\tilde{P} \in \mathcal{M}$.
\end{proof}

\section{Detailed simulation results}
\begin{table}[h]
\centering
\caption{\textbf{Comparison of Log Losses, Square Losses, and MCMC convergence statistics ($\hat{R}$, ESS) across geometries.} Bold values indicate the best performing model for each dataset.}
\label{tab:results_summary}
\small
\begin{tabular}{ll c rr rrrr}
\toprule
\textbf{Dataset} & \textbf{Geometry} & $\eta$ & \textbf{Log Loss} & \textbf{Sq Loss} & ESS & $\hat{R}$ \\
\midrule
\addlinespace[0.5em]
Poincare & Spherical & 0.3 & 0.951 & 0.250 & 905 & 1.01 \\
 & Euclidean & 0.5 & 0.877 & 0.242 & 580 & 1.01 \\
 & Hyperboloid & 0.2 & \textbf{0.562} & 0.166 & 14 & 1.20 \\
\addlinespace[0.5em]
Spherical & Spherical & 0.8 & \textbf{0.466} & 0.134 & 7 & 1.51 \\
 & Euclidean & 0.6 & 1.104 & 0.297 & 6 & 1.75 \\
 & Hyperboloid & 0.8 & 0.852 & 0.226 & 311 & 1.02 \\
\addlinespace[0.5em]
Core-Periphery & Spherical & 0.1 & 1.829 & 0.514 & 840 & 1.00 \\
 & Euclidean & 0.1 & \textbf{1.784} & 0.507 & 1997 & 1.00 \\
 & Hyperboloid & 0.1 & 1.827 & 0.505 & 1350 & 1.00 \\
\addlinespace[0.5em]
Windsurfers & Spherical & 0.3 & 2.203 & 0.717 & 581 & 1.00 \\
 & Euclidean & 0.3 & 2.179 & 0.711 & 805 & 1.00 \\
 & Hyperboloid & 0.1 & \textbf{2.156} & 0.709 & 335 & 1.01 \\
\addlinespace[0.5em]
5-SBM & Spherical & 0.2 & 0.921 & 0.258 & 219 & 1.01 \\
 & Euclidean & 0.4 & \textbf{0.824} & 0.244 & 327 & 1.02 \\
 & Hyperboloid & 0.2 & 0.837 & 0.248 & 280 & 1.01 \\
\addlinespace[0.5em]
Moreno train 9/11 & Spherical & 0.2 & 1.169 & 0.304 & 1897 & 1.00 \\
 & Euclidean & 0.3 & 1.152 & 0.288 & 901 & 1.00 \\
 & Hyperboloid & 0.1 & \textbf{1.070} & 0.257 & 141 & 1.03 \\
\addlinespace[0.5em]
Karate & Spherical & 0.1 & \textbf{1.933} & 0.574 & 1517 & 1.00 \\
 & Euclidean & 0.1 & 1.972 & 0.579 & 2162 & 1.00 \\
 & Hyperboloid & 0.1 & 2.071 & 0.606 & 11 & 1.27 \\
\addlinespace[0.5em]
Iceland & Spherical & 0.1 & 0.762 & 0.166 & 1939 & 1.00 \\
 & Euclidean & 0.1 & \textbf{0.739} & 0.164 & 3668 & 1.00 \\
 & Hyperboloid & 0.1 & 0.827 & 0.165 & 1519 & 1.01 \\
\addlinespace[0.5em]
Les Miserables & Spherical & 0.2 & 1.028 & 0.247 & 1391 & 1.00 \\
 & Euclidean & 0.2 & 0.995 & 0.252 & 1571 & 1.00 \\
 & Hyperboloid & 0.1 & 0.940 & 0.228 & 502 & 1.01 \\
\addlinespace[0.5em]
Dolphins & Spherical & 0.2 & \textbf{1.071} & 0.290 & 1231 & 1.00 \\
 & Euclidean & 0.3 & 1.079 & 0.291 & 1268 & 1.00 \\
 & Hyperboloid & 0.1 & 1.141 & 0.300 & 1120 & 1.00 \\
\bottomrule
\end{tabular}
\end{table}



\end{document}